\documentclass[letterpaper, 10 pt, conference]{packages/ieeeconf}
\IEEEoverridecommandlockouts
\usepackage[margin=0.75in]{geometry}
\usepackage{packages/algorithm}
\usepackage{packages/algorithmic}

\usepackage[compress]{cite}
\makeatletter
\let\NAT@parse\undefined
\makeatother
\usepackage[colorlinks=true]{hyperref}

\usepackage[table]{xcolor} 
\usepackage[section]{placeins} 

\usepackage{graphics} 
\usepackage{amsmath} 

\usepackage{mathtools} 
\usepackage{todonotes}
\usepackage{amssymb}  
\usepackage{makeidx} 
\usepackage{graphicx} 
\graphicspath{{figures/}}
\usepackage{multicol} 
\usepackage[bottom]{footmisc} 
\usepackage[utf8]{inputenc} 
\usepackage{multirow}
\usepackage{booktabs} 
\usepackage{adjustbox}
\usepackage{caption}
\usepackage{makecell} 
\usepackage{subcaption}
\usepackage{tikz}
\usetikzlibrary{positioning}
\usetikzlibrary{shapes,snakes}
\usepackage{bm} 
\usepackage{float} 
\usepackage{color} 
\usepackage{empheq} 
\usepackage{threeparttable} 
\usepackage{xcolor}
\usepackage{siunitx}
\usepackage{xurl} 

\usepackage{paralist}
\usepackage{acronym}  

\acrodef{DOF}{degrees of freedom}
\acrodef{SNR}{signal-to-noise ratio}
\acrodef{RMSE}{root mean squared error}
\acrodef{GT}{ground truth}

\newcommand\scalemath[2]{\scalebox{#1}{\mbox{\ensuremath{\displaystyle #2}}}}

\usepackage{mathtools}
\DeclarePairedDelimiterX{\norm}[1]{\lVert}{\rVert}{#1}

\newcommand{\diag}{\text{diag}}
\newcommand{\skewm}[1]{\ensuremath{\lfloor #1\, \times \rfloor}}
\newcommand{\rotm}{\ensuremath{\mathbf{C}}}
\newcommand{\rotmh}{\ensuremath{\mathbf{\hat{C}}}}

\def\tildebtheta{{\mbox{\boldmath $\tilde{\theta}$}}}

\def\bbzeta{{\mbox{\boldmath $\zeta$}}}
\def\bbomega{{\mbox{\boldmath $\omega$}}}
\def\bbOmega{{\mbox{\boldmath $\Omega$}}}
\def\bbzeta{{\mbox{\boldmath $\zeta$}}}

\newcommand{\eye}[1]{\ensuremath{\mathbf{I}_{#1}}}

\allowdisplaybreaks
\begin{document}

\title{\LARGE \bf
A Direct Algorithm for Multi-Gyroscope Infield Calibration
}

\author{Tianheng Wang and Stergios I. Roumeliotis
%
\thanks{The authors are with Apple Inc., Cupertino, CA 95014, USA.
Email: {\tt\small \{tianheng$\_$wang2|stergios\}@apple.com}}%
}

\maketitle
\thispagestyle{plain} 
\pagestyle{plain}     


\begin{abstract}
In this paper, we address the problem of estimating the rotational extrinsics, as well as the scale factors of two gyroscopes rigidly mounted on the same device.
In particular, we formulate the problem as a least-squares minimization and introduce a {\em direct} algorithm that computes the estimated quantities without any iterations, hence avoiding local minima and improving  efficiency.
Furthermore, we show that the rotational extrinsics are observable while the scale factors can be determined up to global scale for general configurations of the gyroscopes. 
To this end, we also study special placements of the gyroscopes where a pair, or all, of their axes are parallel and analyze their impact on the scale factors' observability.
%
%
Lastly, we evaluate our algorithm in simulations and real-world experiments to assess its performance as a function of key motion and sensor characteristics. 
\end{abstract}

\section{Introduction and Related Work}

Most mobile and portable devices (e.g., robots, cell phones, wearables) are equipped with an inertial measurement unit (IMU) for estimating its motion, often in conjunction with additional sensors, such as cameras, laser scanners, ultrawideband, or ultrasound sensors~\cite{Mourikis2007_ICRA,Li2013_IJRR,leutenegger2015_IJRR,Wu2015RSS,Tien2019IJRR,Qin2018_TRO,Geneva2020ICRA,Lee2021ICRA,RahmanIJRR2022,Zuo2019IROS,Zuo2020IROS,Zheng2022,Gao2022}.
An IMU typically comprises a 3-axial accelerometer and a 3-axial gyroscope each measuring linear accelerations and rotational velocities, respectively, expressed in its own frame of reference.
Using more than one IMU has the potential to increase motion estimation accuracy by fusing their measurements after accounting for their {\em extrinsics}~\cite{Eckenhoff2021TRO,Zhang2020RAL,Yang2023arXiv}.

Multi-IMU extrinsics can be determined in a number of ways.\footnote{Our focus  is on multi-IMU calibration. For single-IMU calibration, the interested reader is referred to~\cite{Zhang2014_SJ,Tedaldi2014ICRA,Yang2023TRO} and references therein.}
In certain cases, IMUs are calibrated for both their extrinsics and intrinsics by employing specialized equipment such as a rate table~\cite{Cho2005,rohac2015}, in conjunction with a camera and using a calibration target~\cite{rehder2016extending}, or an external optical system~\cite{Jadid2019}.
%
Although such approaches can provide excellent calibration accuracy, they have high cost, due to their dependence on hardware infrastructure.
Additionally, their results are valid potentially for {\em limited time}, as both extrinsics and intrinsics may change with time and/or device use due to, e.g., humidity, temperature, and mechanical stress.  
For this reason, recent research has focused on {\em infield} calibration of IMUs.
When additional sensors, such as a camera, are available, the IMUs extrinsics can be estimated along with the device’s pose as part of a visual-inertial algorithm~\cite{Eckenhoff2021TRO}.
Such approaches, however, have a number of limitations. 
Firstly, the processing cost can be quite high due to the dependence on the aiding camera(s) and the requirement to extract, match, and process visual features within the visual-inertial estimator.
Secondly, these methods can only be used when additional sensors are present.
Moreover, even when the aiding sensor, such as a camera, is available, the operating conditions (e.g., lighting, distance to the scene, texture, motion) may be unfavorable and negatively impact the calibration quality.
Lastly, a joint localization, mapping, and calibration problem is typically high dimensional, nonlinear and thus there are no guarantees for convergence to the global minimum,


Besides the potential of multiple IMUs to increase the accuracy of visual-inertial state estimation, there exist other use cases, where the key parameters to 
%
be estimated and monitored are the {\em rotational extrinsics between the gyroscopes}.
%
Specifically, gyroscopes can be attached on the back of the binocular displays of a head-mounted device, or on ultrasound sensors installed on an underwater robot. 
In these applications, tracking the extrinsics of the gyroscopes for changes due to, e.g., temperature or shock, allows us to compensate for binocular disparity in the displays ~\cite{DisplaySystems2015,SelfTR1986}, or account for changes in the viewing directions of exteroceptive sensors with non-overlapping fields of view~\cite{RahmanIJRR2022}.
%
%
Achieving this without using any additional sensors besides IMUs is highly desirable especially when considering operating in challenging  conditions or under processing constraints.
%

Alternative approaches that directly calibrate the extrinsics of pairs of IMUs include the works of~\cite{Lee2022RAL} and~\cite{Yu2023ICRA} where rotation between the gyroscopes and translation of the accelerometers are estimated by assuming {\em known and constant} intrinsics (scale factors and skewness).
This assumption, however, is often violated in practice especially for the {\em scale factors} of gyroscopes found in MEMS IMUs, which drift due to temperature, humidity, shock, and mechanical strain.
%

%
%

To address the limitations of previous approaches, in this paper we make the following main contributions:
\begin{enumerate}
    \item We formulate the problem of estimating the rotational extrinsics and intrinsics of a pair of gyroscopes as  least-squares (LS) and introduce an algorithm for {\em directly} computing the observable subset of them.
    
    \item We show that the extrinsics and  scale factors, up to a global scale, are, in general, observable. Additionally, we consider degenerate cases of gyroscope placement and analyze their impact on scale factor observability.
    
    %
    %
    \item We evaluate the presented algorithm in simulations and experiments to assess its performance as a function of key motion and sensor parameters. Additionally, we evaluate the theoretically predicted impact of skewness error on the resulting calibration accuracy. Lastly, we consider the effect of device flexibility and propose a method for detecting and rejecting rotation data during the times that the rigidity assumption is violated.
    
\end{enumerate}

\section{Problem Formulation}
\label{sec:method}

Our objective is to use angular velocity measurements from two gyroscopes to estimate their 3~degrees of freedom (d.o.f.) rotational extrinsics and scale factors. 
In what follows, we assume that the gyroscopes' data are properly time stamped and re-sampled\footnote{Any unknown time offset between the gyroscopes is typically computed by correlating their measurements' {\em norms}, and then compensated for via time alignment. Resampling can be done via cubic-spline interpolation.} to ensure that their measurements correspond to the same time instants.

\subsection{Problem Formulation}
Each of the gyroscopes' measurements, ${}^{i} \bbomega_m(k), ~i=1,2$, at time-step $k$ can be written as
\begin{eqnarray} \label{eq:omega_meas}
{}^{i} \bbomega_m(k) = \mathbf{S}_i  ~{}^{i} \bbomega(k) + \mathbf{b}_i(k) + \mathbf{n}_i(k) 
\end{eqnarray}
where ${}^{i} \bbomega(k) $ is the true rotational velocity expressed w.r.t. frame $\{ i \}$,  $\mathbf{b}_i(k)$ and $\mathbf{n}_i(k)$ are the bias and noise for the $i$th gyroscope, respectively, and $ \mathbf{S}_i $ is the {\em upper-triangular} intrinsic matrix that represents the impact of scale and skewness.\footnote{Note that scale is often referred to as {\em sensitivity}, while skewness is also known as {\em cross-axis sensitivity} and typically represents the effect of non-orthogonality of the gyroscopes' axes.}
Since the two gyroscopes are rigidly attached to the same body, the true rotational velocities they measure are related through the rotational matrix $\mathbf{C} = {}^{1}_{2}\mathbf{C}$ that transforms vectors expressed w.r.t. frame $\{ 2\}$ to frame $\{ 1\}$, i.e.,
\begin{eqnarray}  \label{eq:omega_C}
{}^{1} \bbomega(k)  = \mathbf{C}  ~{}^{2} \bbomega(k)\,.
\end{eqnarray}
Solving \eqref{eq:omega_meas} for ${}^{i} \bbomega(k), ~i=1,2,$ and substituting in \eqref{eq:omega_C}, yields
\begin{align}
{}^{1} \bbomega_m(k)  &=  \mathbf{A} {}^{2} \bbomega_m(k)  + \mathbf{b}(k)  + \mathbf{n}(k)
\end{align}
where we have defined
\begin{align} \label{eq:defineA}
\mathbf{A} &\coloneqq  \mathbf{S}_1 \mathbf{C}   \mathbf{S}_2 ^{-1}  \\
\label{eq:defineb}
\mathbf{b}(k)  &\coloneqq \mathbf{b}_1(k) -  \mathbf{A}  \mathbf{b}_2(k)  \\
\mathbf{n}(k)  &\coloneqq  \mathbf{n}_1(k)  -  \mathbf{A}  \mathbf{n}_2(k)\,. 
\end{align}
Here $\mathbf n_1(k)$ and $\mathbf n_2(k)$ are considered to be white Gaussian noise sequences with $\text{Cov} \{\mathbf n_i(k)\} = \sigma_i^2\mathbf I$. Since $\mathbf A$ is close to orthonormal matrix, we have
\begin{align}
    \text{Cov}\{\mathbf n(k)\} \simeq \sigma^2\mathbf I~~,\, ~~~\sigma^2 = \sigma_1^2 + \sigma_2^2.
\end{align}
Hereafter, we assume that the bias in each gyro is approximately constant for the duration of the experiment and ignore its dependence on time, which allows us to form the optimization problem:
\begin{eqnarray} \label{eq:cost}
\min_{\mathbf A, \mathbf b}~ \mathcal{C}(\mathbf A, \mathbf b) = \frac{1}{2\sigma^2} \sum_{k=1}^{N} \left\| {}^{1} \bbomega_m(k)  - \mathbf{A} {}^{2} \bbomega_m(k)  - \mathbf{b}  \right\|^2
\end{eqnarray}
%

\noindent
At this point, we make the following key observations:
\begin{enumerate}
    \item The optimization problem~\eqref{eq:cost} is {\em quadratic} in the {\em combined} bias $\mathbf{b}$  and  matrix $\mathbf A$. This motivates our approach to first {\em directly} solve for $\mathbf A$, and then decompose it into the gyroscopes' intrinsics and extrinsics [see~~\eqref{eq:defineA}].

    \item The time steps from which we use rotational velocity measurements to form the cost function in~\eqref{eq:cost} need {\em not} be consecutive. Instead, as shown in Section~\ref{sec:simulations}, one should select measurements from time instants when the rotational norm is large (and thus improve the signal to noise ratio), but also avoid periods of very fast motion where the rigidity assumption is violated.
    \item The $9$ elements of matrix $\mathbf{A}$ are functions of the $15$ unknown calibration d.o.f. corresponding to the $6$ d.o.f of each of the gyroscopes' intrinsics in $\mathbf{S}_i$, $i=1,2$ and $3$ d.o.f. of the rotation matrix $\mathbf C$.
\end{enumerate}

Due to the limited information available to the two gyroscopes when in motion, potentially up to 9 out of the 15 unknown d.o.f in $\mathbf{A}$ can be determined. 
%
Besides the extrinsics, the gyroscopes' scale factors are usually susceptible to drift.
In contrast, the skewness remains approximately constant over a wide range of conditions, and its impact on accuracy can be theoretically predicted as confirmed in Section~\ref{sec:exp_skewness}.
%
%
%
For this reason, from here on we assume that the gyroscopes' measurements have been compensated for skewness using the values provided by the manufacturer, and model the intrinsics as a diagonal matrix of scale factors:
\begin{align} \label{eq:S_i}
    \mathbf{S}_i  &\coloneqq \diag\{ {}^{i}s_x, {}^{i}s_y, {}^{i}s_z \}\,.
\end{align}
This reduces the dimensionality of the problem from $15$ to $9$ and makes a potential decomposition of matrix $\mathbf A$ into its factors possible.
Note, however, that each of the 9 elements of $\mathbf A$ is a {\em $4$th-order polynomial} in the unknowns.
Specifically, $\mathbf A$ is a bilinear function of the scale-factor matrices $\mathbf{S}_1$ and $\mathbf{S}_2^{-1}$. Additionally, the elements of the orientation matrix $\mathbf{C}$ are $2$nd-order polynomials of the orientation expressed using, e.g., Cayley or quaternion parameterization~\cite{Craig2006Book} .
%
%
%
The optimization problem in~\eqref{eq:cost} could be solved by linearizing $\mathbf A$ in (\ref{eq:defineA}) w.r.t. the elements of $\mathbf{S}_i$, $i=1,2$ and $\mathbf{C}$, and iteratively minimizing the cost function using, e.g., Gauss-Newton. 
%
%
Such approach, however, is susceptible to local minima and may require many iterations to converge.
For this reason, in the next section, we present the main contribution of this paper: a \textit{direct solver} that is able to determine the global minimum of~\eqref{eq:cost} in a single step.




\subsection{Direct Least Squares Algorithm}
%


The optimization problem (\ref{eq:cost}) can be solved as follows
\begin{align}
    \min_{\mathbf A,\mathbf b}~\mathcal C(\mathbf A,\mathbf b)
    &\Leftrightarrow
    \min_{\mathbf A}\,\min_{\mathbf b | \mathbf A}~\mathcal C(\mathbf A,\mathbf b) \label{eq:optimization}
\end{align}
where $\min_{\mathbf b | \mathbf A}$ means minimizing over $\mathbf b$ given $\mathbf A$. The inner problem can be solved by setting the derivative of $\mathcal{C}$ in (\ref{eq:cost}) with respect to $\mathbf{b}$ equal to zero, which yields:
\begin{align}
\nabla_{\mathbf{b}} \mathcal{C} = 0 &\Rightarrow \sum_{k=1}^{N} \left({}^{1} \bbomega_m(k)  - \mathbf{A} {}^{2} \bbomega_m(k)  - \mathbf{b}  \right) = 0\\
 &\Rightarrow
 \mathbf{b^*}(\mathbf A)
 %
 =  {}^{1} \bar{ \bbomega}_m - \mathbf{A}   {}^{2} \bar{ \bbomega}_m \label{eq:bias}
\end{align}
where
%
%
%
\begin{align}
 {}^{i} \bar{ \bbomega}_m  &\coloneqq  \frac{1}{N}  \sum_{k=1}^{N} {}^{i} \bbomega_m(k)\,,~~~ i=1,2\,.
\end{align}
After substituting $ \mathbf{b} = \mathbf{b^*}(\mathbf A)$ from \eqref{eq:bias} into \eqref{eq:optimization}, we have
\begin{align}
    \min_{\mathbf A,\mathbf b}~\mathcal C(\mathbf A,\mathbf b)
    &=
    \min_{\mathbf A}~\mathcal C(\mathbf A,\mathbf b^*(\mathbf A))
    =
    \min_{\mathbf A}~\mathcal C'(\mathbf A)
\end{align}
where
\begin{align} \label{eq:cost2}
\mathcal C'(\mathbf A) &= \frac{1}{2\sigma^2} \sum_{k=1}^{N} \left\| {}^{1} \breve{\bbomega}_m(k)  - \mathbf{A} {}^{2} \breve{\bbomega}_m(k)    \right\|^2
\end{align}
in which
\begin{align} \label{eq:omega_shift}
 {}^{i} \breve{\bbomega}_m(k) &\coloneqq  {}^{i} \bbomega_m(k)  - {}^{i} \bar{ \bbomega}_m   ~~,~ i=1,2\,.
\end{align}
In order to minimize \eqref{eq:cost2} over $\mathbf A$, we form the following matrix equation
\begin{align}
\left[ \, {}^{1} \breve{\bbomega}_m(1) \, \dots \,  {}^{1} \breve{\bbomega}_m(N) \,\right] &= \mathbf{A}  \left[ \, {}^{2} \breve{\bbomega}_m(1) \, \dots \,  {}^{2} \breve{\bbomega}_m(N) \,\right]  \nonumber
\\
%
\Rightarrow ~~\bigl( \bbOmega_2  \bbOmega_2^\top   \bigr) \cdot \mathbf{A^*}  &= \bbOmega_1 \bbOmega_2^\top  \label{eq:optimalA}
\end{align}
where
\begin{align}
\bbOmega_i &\coloneqq \left[\,  {}^{i} \breve{\bbomega}_m(1) ~ \dots ~  {}^{i} \breve{\bbomega}_m(N)\, \right] ~~,~ i=1,2\,.  \label{eq:Omega_i}
\end{align}
\newtheorem{lemma}{Lemma}
\begin{lemma}\label{lem1}
To recover the matrix $\mathbf{A}$ 
in~\eqref{eq:optimalA}, 
1) at least 4 pairs of measurements that span $3$ d.o.f. need to be available, 
and 2) assuming that ${}^2\bbomega_m(k)$, $k=1,2,3$, span $3$ d.o.f.,
for ${}^2\bbomega_m(4) = \sum_{k=1}^3 \alpha_k {}^2\bbomega_m(k)$, $\sum_{k=1}^3 \alpha_k \neq 1$ is required.
\end{lemma}
\begin{proof}
See Appendix \ref{lem1_proof}.
\end{proof}
%
%
%

Once we have computed the optimal value of $\mathbf{A}$ [from \eqref{eq:optimalA}], we employ \eqref{eq:defineA} to recover the scale-factor matrices $\mathbf{S}_1$ and $\mathbf{S}_2$ for the two gyros, as well as the rotational matrix $\mathbf{C}$ between them. 
To do so, we first express \eqref{eq:defineA} w.r.t. $\mathbf{C}$:
\begin{eqnarray} \label{eq:CfromA}
\mathbf{C} = \mathbf{S}_1^{-1} \mathbf{A}   \mathbf{S}_2  
\end{eqnarray}
and employ the orthonormality constraint:
\begin{eqnarray} \label{eq:C_identity}
\mathbf{C} \mathbf{C} ^\top = \mathbf{I}_3 
~\Rightarrow~
 \mathbf{A}   \mathbf{S}_2^2  \mathbf{A} ^\top = \mathbf{S}_1^2 \,.
\end{eqnarray}
Defining the matrix $\mathbf{A}$ element-wise as $\mathbf{A} \coloneqq \left[ a_{ij} \right], ~i,j=1,2,3,$ and substituting from \eqref{eq:S_i}, \eqref{eq:C_identity} can be re-written as:
\begin{eqnarray} \label{eq:s_is}
\begin{bmatrix}
\mathbf{A}_I & -\mathbf{I}_3\\
\mathbf{A}_0 & \mathbf{0}_{3 \times 3}
\end{bmatrix}
\begin{bmatrix}
\mathbf s_2\\
\mathbf s_1
\end{bmatrix}
=
\begin{bmatrix}
\mathbf{0}_{3 \times 1}\\
\mathbf{0}_{3 \times 1}
\end{bmatrix} 
\end{eqnarray}
where the vectors $\mathbf{s}_1$ and $\mathbf{s}_2$ comprise the diagonal elements of the matrices  $\mathbf{S}_1^2$ and $\mathbf{S}_2^2$ [see \eqref{eq:S_i}], respectively, i.e.,
\begin{eqnarray} \label{eq_sis}
\mathbf s_i \coloneqq \begin{bmatrix} {}^{i}s_{x}^2 & {}^{i}s_{y}^2 & {}^{i}s_{z}^2
\end{bmatrix}^\top ~~,~ i=1,2 
\end{eqnarray}
%
and
\begin{align} \label{eq:AI}
\mathbf{A}_I &\coloneqq \begin{bmatrix}
a_{11}^2 & a_{12}^2 & a_{13}^2 \\
a_{21}^2 & a_{22}^2 & a_{23}^2 \\
a_{31}^2 & a_{32}^2 & a_{33}^2 
\end{bmatrix}
\\ \label{eq:A0}
\mathbf{A}_0 &\coloneqq \begin{bmatrix}
a_{11} a_{21}  & a_{12} a_{22}  & a_{13} a_{23}\\
a_{21} a_{31}  & a_{22} a_{32}  & a_{23} a_{33}\\
a_{31} a_{11}  & a_{32} a_{12}  & a_{33} a_{13}
\end{bmatrix}.
\end{align}

\subsection{Computational Complexity}
At this point, we should note that the proposed algorithm has linear complexity $\mathcal O(N)$ in the number of gyro samples $N$; this is due to~\eqref{eq:optimalA} where the product of two pairs of $3\times N$ matrices is computed before solving for the $3 \times 3$ matrix $\mathbf{A}$.
The rest of the operations involve inverting or finding the null space of $3\times 3$ matrices, and thus have constant cost.





\section{Observability Analysis} \label{sec:determination}

Hereafter, we turn our attention to the issue of determining how many d.o.f. are observable when considering two special and the general case of the gyroscopes'
 relative orientation. 

\begin{lemma} \label{lem2}
If the two gyroscopes are placed with their axes parallel to each other, then only 6 d.o.f. are observable: 3 for orientation and 3 for the pairwise ratios of the scale factors. 
\end{lemma}
\begin{proof}
When the gyroscopes are placed with their axes parallel, the rotation matrix $\mathbf{C}$ becomes a {\em signed} permutation matrix and, hence, $\mathbf{A}$ a generalized permutation matrix. 
In this case, it is easy to see that $\mathbf{A}_I$ in~\eqref{eq:AI}, which comprises the squares of the elements of $\mathbf{A}$, will also be a generalized permutation matrix of rank 3, while $\mathbf{A}_0$ in~\eqref{eq:A0} will be $\mathbf{0}_{3\times 3}$, i.e., of rank~$0$, as each row/column of $\mathbf{A}$ contains only one non-zero element.
Thus,~\eqref{eq:s_is} becomes:
\begin{eqnarray}
    \mathbf{A}_I \cdot \mathbf{s}_2 = \mathbf{s}_1
\end{eqnarray}
from which we can only recover the pairwise ratios of the squared values of the scale factors corresponding to each of the gyroscopes parallel axes.
Lastly, the 3 d.o.f. of the extrinsics matrix $\mathbf{C}$ are computed by converting $\mathbf{A}$ into a signed permutation matrix.
\end{proof}
%
%
\begin{lemma}\label{lem3}
    If the two gyroscopes are placed so that only one pair of their axes are parallel,  then 7~d.o.f. are observable: 3 for the gyroscopes' orientation, 1 for the ratio of the scale factors corresponding to the common axis $\mathbf{e}_j$, and 3 for the remaining 4 scale factors.  
\end{lemma}
\begin{proof} 
Without loss of generality, we consider the case when the gyroscopes' $z$ axes are parallel and point in the same direction, i.e., their relative orientation corresponds to $\mathbf{C} = \mathbf{C}(\mathbf{e}_3, \theta)$,  where $\left[ \mathbf{e}_1~\mathbf{e}_2~\mathbf{e}_3\right] = \mathbf{I}_3$ and $\theta \in \mathbb{R}^*$. The analysis of all other cases is analogous after permuting the gyroscopes' axes so that their new $z$ axes are aligned.\footnote{For example, the case when ${}^{1}x \parallel {}^{2}y$, and point in opposite directions, can be mapped to the one detailed here by first pre-multiplying the measurements of gyroscopes~1 and~2 with the signed permutation matrices $\mathcal{}{P}_1 = \left[ \mathbf{e}_3~\mathbf{e}_2~\mathbf{e}_1~\right]$
 and $\mathcal{}{P}_2 = \left[ \mathbf{e}_1~-\mathbf{e}_3~-\mathbf{e}_2~\right]$, respectively.}


Since $\mathbf{A}$ has the same zeros structure as $\mathbf{C}$ [see~\eqref{eq:defineA}], the non-diagonal elements corresponding to the row/column of the axis of rotation $\mathbf{e}_j$ will be zero. 
Analogously, $\mathbf{A}_I$ has zeros at the same locations as $\mathbf{C}$, while $\mathbf{A}_0$ will have a right null space of rank one.
In particular, if $\mathbf{C} = \mathbf{C}(\mathbf{e}_3, \theta)$, then
\begin{align*} 
\mathbf{A}_I = \begin{bmatrix}
a_{11}^2 & a_{12}^2 & 0 \\
a_{21}^2 & a_{22}^2 & 0 \\
0 & 0 & a_{33}^2 
\end{bmatrix}
,~
\mathbf{A}_0 = \begin{bmatrix}
a_{11} a_{21}  & a_{12} a_{22}  & 0\\
0  & 0  & 0\\
0  & 0  & 0
\end{bmatrix}.
\end{align*}
Substituting in~\eqref{eq:s_is} and re-arranging terms yields a relation for the ratio of the $z$-axes' scale factors:
\begin{align} \label{eq:z_scale}
     a_{33}  = {}^{1}s_{z} / {}^{2}s_{z}
\end{align}
and one for the scale factors of the remaining four axes
\begin{align}
\begin{bmatrix}
a_{11}^2 & a_{12}^2 &  -1 & 0\\
a_{21}^2 & a_{22}^2 &  0 & -1\\
a_{11} a_{21} & a_{12} a_{22} &  0 & 0
\end{bmatrix}
\begin{bmatrix}
    {}^{2}s_{x}^2 \\ {}^{2}s_{y}^2 \\ {}^{1}s_{x}^2 \\ {}^{1}s_{y}^2
\end{bmatrix} = \mathbf{0}_{4 \times 1}\,.
\end{align}
As evident, the rows of the preceding matrix are linearly independent and hence the above equation has a unique right null space $\bbzeta$ from which the 4 remaining scale factors can be determined up to a global scale $\lambda \in \mathbb{R}^*$
\begin{align} \label{eq:xy_scale}
    \begin{bmatrix}
    {}^{2}s_{x} & {}^{2}s_{y} & {}^{1}s_{x} & {}^{1}s_{y}
\end{bmatrix}^\top = \lambda \cdot \breve{\bbzeta}
\end{align}
where $\breve{\bbzeta}(i) = \sqrt{\bbzeta(i)}$, $i=1, \hdots, 4$.
Lastly, substituting from~\eqref{eq:xy_scale} and~\eqref{eq:z_scale} in~\eqref{eq:CfromA}, yields the rotational matrix
\begin{align}
    \mathbf{C}(\mathbf{e}_3,\theta) = 
    \begin{bmatrix}
        a_{11} \frac{\breve \bbzeta(1)}{\breve \bbzeta(3)} & a_{12} \frac{\breve \bbzeta(2)}{\breve \bbzeta(3)} & 0 \\
        a_{21} \frac{\breve \bbzeta(1)}{\breve \bbzeta(4)} & a_{22} \frac{\breve \bbzeta(2)}{\breve \bbzeta(4)} & 0 \\
        0 & 0 & 1         
    \end{bmatrix}
\end{align}
with $\theta = {\rm atan2}(a_{21} \breve \bbzeta(1), a_{22} \breve \bbzeta(2))$.
\end{proof}
%
%
\begin{lemma}\label{lem4}
    If the two gyroscopes are placed so that none of their axes are parallel, then 8 d.o.f. are observable: 3 for the relative orientation and 5 for the gyroscopes' scale factors.
\end{lemma}
\begin{proof}
In the general case, where none of the elements of $\mathbf{C}$ is zero, matrix $\mathbf{A}_0$ has rank 2.\footnote{Due to numerical errors, the matrix $\mathbf{A}_0$ is almost never rank deficient. Thus, we typically compute $\mathbf{s}_2^* $ as the eigenvector corresponding to the smallest eigenvalue of  matrix $\mathbf{A}_0^\top\mathbf{A}_0$. } Note also, that $\mathbf{A}_0$ cannot be of rank 3, as this would result in $\mathbf{s}_2=\mathbf{0}_{3 \times 1}$ and $\mathbf{s}_1=\mathbf{0}_{3 \times 1}$, which is impossible.
From \eqref{eq:s_is}, we have $\mathbf A_0 \mathbf s_2 = \mathbf 0_{3\times 1}$ which allows us to compute the optimal value of  $\mathbf s_2$ and subsequently of $\mathbf s_1$ as:
    \begin{eqnarray}
    \mathbf s_2^* = \mathrm{null}\left(\mathbf{A}_0\right) ~~~~,\,~~~~ \mathbf s_1^* = \mathbf{A}_I \mathbf s_2^*\,.
    \end{eqnarray}
    Note that since $\mathbf s_2^*$ is a {\em unit vector}, the {\em absolute} scale for the gyros cannot be recovered.
Once $\mathbf s_1^*$ and $\mathbf s_2^*$ have been determined, we compute the corresponding $\mathbf{S}_1^*$ and $\mathbf{S}_2^*$ from \eqref{eq_sis} and \eqref{eq:S_i}.
Subsequently, we recover the unknown rotational matrix $\mathbf{C}$ from \eqref{eq:CfromA}, as
\begin{eqnarray} \label{eq:Copt}
\mathbf{C}^* = \mathbf{S}_1^{*-1} \mathbf{A}^*   \mathbf{S}_2^* \,.
\end{eqnarray}
Lastly, we ensure that $\mathbf{C}^* \in \mathrm{SO(3)}$ by finding the nearest orthonormal matrix to $\mathbf{C}^*$ as in \cite{horn1988closed}. 
\end{proof}
Summarizing the previous four Lemmas, we have shown that
\newtheorem{theorem}{Theorem}
\begin{theorem}
The relative orientation of two gyroscopes is always observable when 1) at least 4 pairs of measurements that span $3$ d.o.f. are available, and 2) assuming that ${}^2\bbomega_m(k)$, $k=1,2,3$, span $3$ d.o.f.,
for ${}^2\bbomega_m(4) = \sum_{k=1}^3 \alpha_k {}^2\bbomega_m(k)$, $\sum_{k=1}^3 \alpha_k \neq 1$ is required.
\end{theorem}
\begin{theorem} \label{thm:observability}
The scale factors of two gyroscopes are  observable, up to a global scale, except for special cases:
\begin{itemize}
    \item If two of the gyroscopes' axes are parallel: 4 scale factors are observable, up to a global scale, and only the ratio of the remaining two, corresponding to the parallel axes, can be determined; 
    \item If all of the gyroscopes' axes are parallel, only the pairwise ratios of their scale factors can be determined. 
\end{itemize}
\end{theorem}

%
A key practical implication from Theorem~\ref{thm:observability}, is that when installing multiple gyroscopes on a device, care should be taken to place them so that they have no parallel axes. 
\subsection{Recovering global scale}
%
%
%
As shown in Theorem~\ref{thm:observability}, global scale is not observable without additional information, even when the gyroscopes are not in a special configuration.
To recover scale, one needs to consider either 
 {\em additional sensors} or {\em prior information}.
%
Specifically, if one or more cameras are also used in conjunction with a pair of IMUs in a visual-inertial estimator, then instead of six, only one d.o.f., that of global scale, would need to be estimated, hence reducing the dimensionality of the problem and improving accuracy.
On the other hand, if it is known that particular axes of the gyroscopes, e.g., the $x$ axes, have stable scale factors close to one (i.e., ${}^is_{x}\simeq 1$, $i=1,2$), then global scale can be recovered by solving the following constrained optimization problem:
%
%
%
\begin{align}
    \min_{{}^1s_{x},{}^2s_{x}}~f({}^1s_{x},{}^2s_{x}) 
    &= \bigl({}^1s_{x}-1\bigr)^2 + \bigl({}^2s_{x}-1\bigr)^2 \label{opt_refine_scale}
    \\
    \text{s.t.}~{}^1 s_{x} / {}^2 s_{x} &= \lambda
\end{align}
where $\lambda \coloneqq {}^1 s_{x} / {}^2 s_{x}$ is the scale factor ratio estimated from the direct method of Section~\ref{sec:determination}. The solution is given by
\begin{align}
    {}^1 s_{x}^* &= \frac{\lambda(\lambda+1)}{\lambda^2+1}~,~~
    {}^2 s_{x}^* = \frac{\lambda+1}{\lambda^2+1}\,.
\end{align}
The remaining scale factors ${}^1s_y$, ${}^1s_z$, ${}^2s_y$, and ${}^2s_z$ are recovered by maintaining their ratios w.r.t. ${}^1 s_{x}$ and ${}^2 s_{x}$. In particular
${}^i s_y^* = ({{}^i s_y}/{{}^i s_x }) {}^i s_{x}^*$
and
${}^i s_z^* = ({{}^i s_z}/{{}^i s_x }) {}^i s_{x}^*$,
where the parameters without $*$ are those before optimization.





\begin{figure}[t]
\centering
\includegraphics[trim=0 10mm 0mm 0,clip,width=0.9\columnwidth]{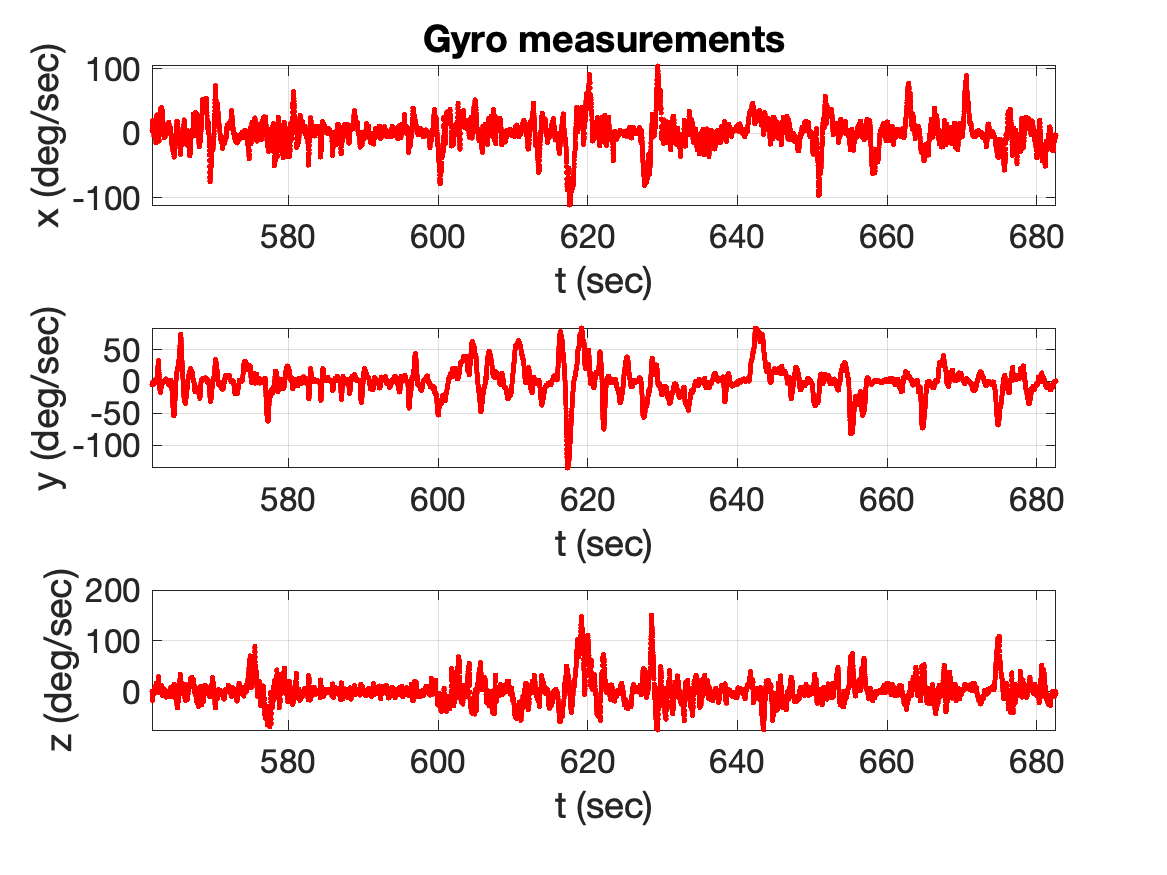}
\caption{Example of gyroscope $1$ measurements.}
\label{fig:gyro_motion}
\end{figure}

\section{Simulation Validation}
\label{sec:simulations}

\begin{figure*}[t]
\centering
\begin{minipage}{0.325\linewidth}
\centering
\includegraphics[trim=0 0mm 0mm 0,clip,width=\linewidth]{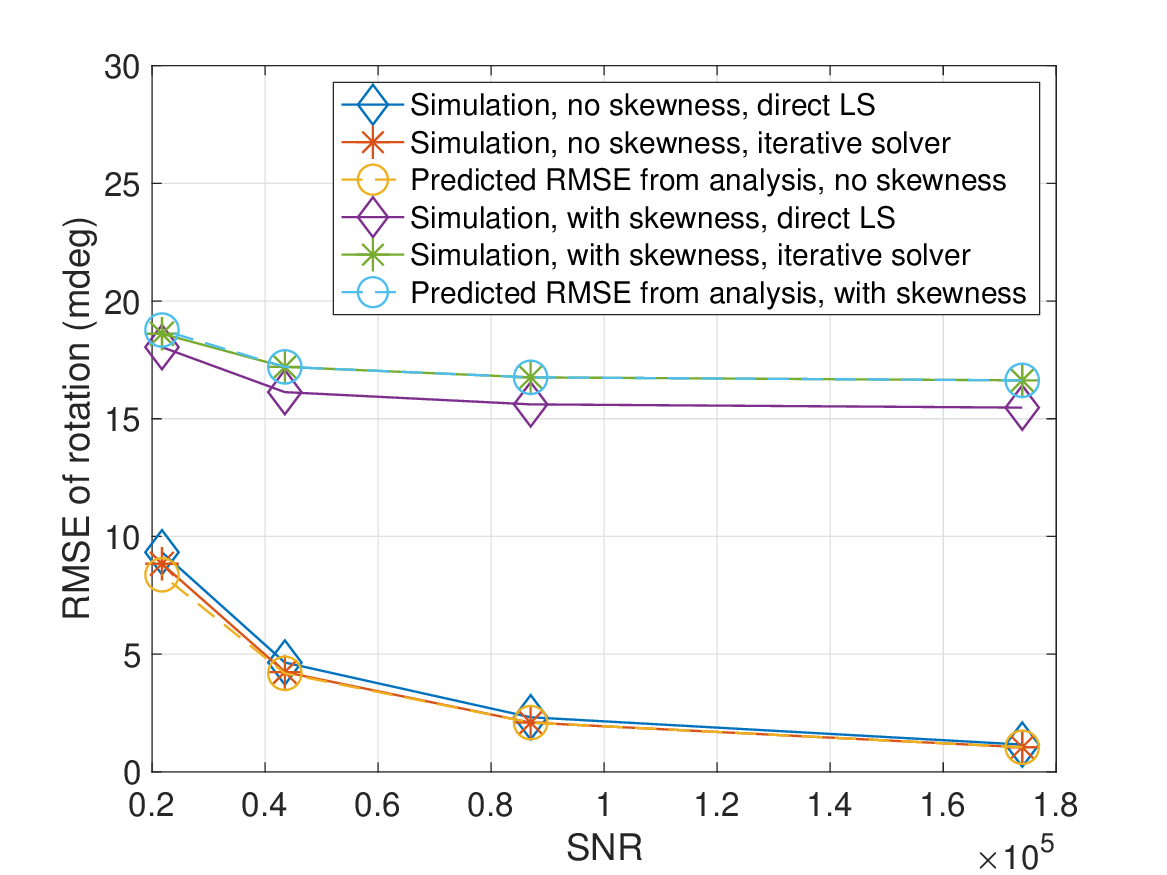}
\captionsetup{margin=0.5em}
\caption{
Rotation \acs{RMSE} w.r.t. \acs{SNR}. Circles represent the predicted \acs{RMSE} due to skewness \eqref{rot_err_analytical}.
}
\label{fig:rmse_rotation}
\end{minipage}
\hfill
\begin{minipage}{0.325\textwidth}
\includegraphics[trim=0 0mm 0mm 0,clip,width=\linewidth]{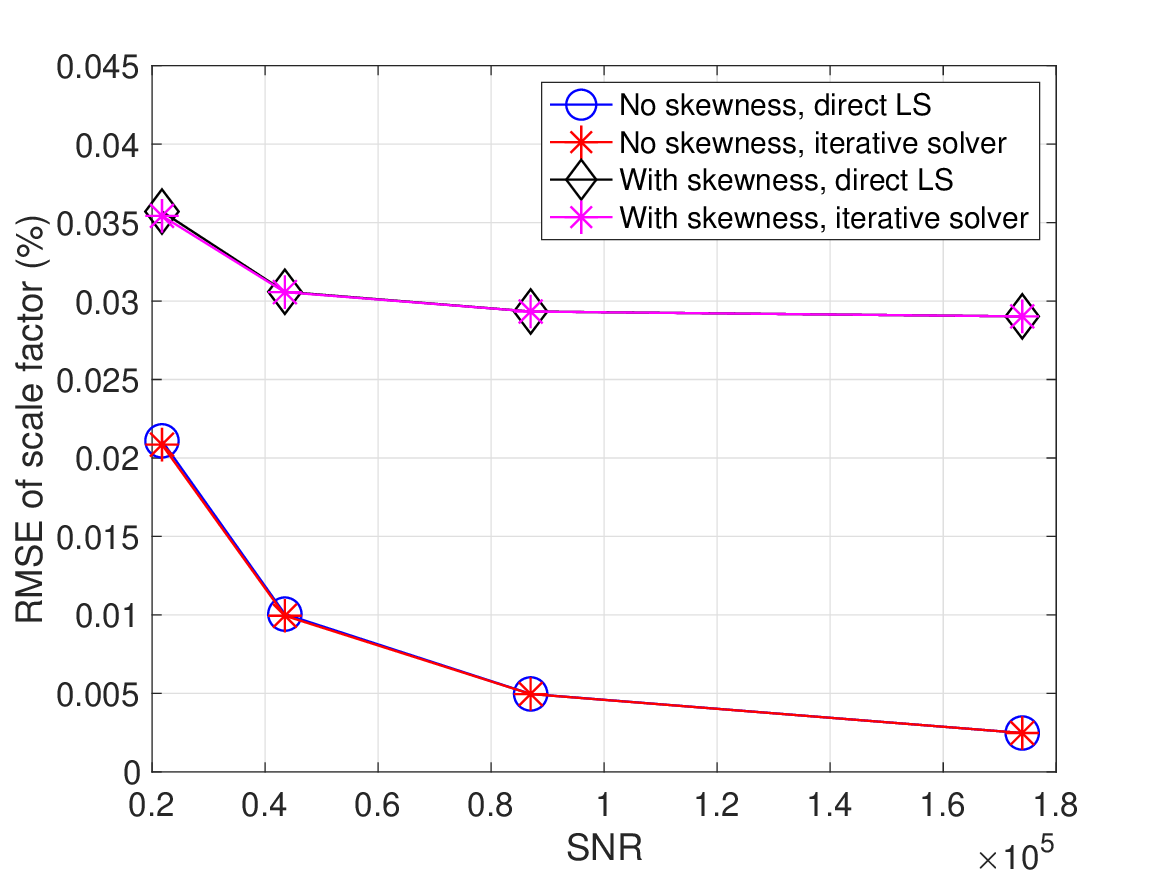}
\captionsetup{margin=0.5em}
\caption{
Scale factor \acs{RMSE} (\%) w.r.t. \acs{SNR}. \newline
}
\label{fig:rmse_scale}
\end{minipage}
\hfill
\begin{minipage}{0.325\linewidth}
\includegraphics[trim=0 0mm 0mm 0,clip,width=\linewidth]{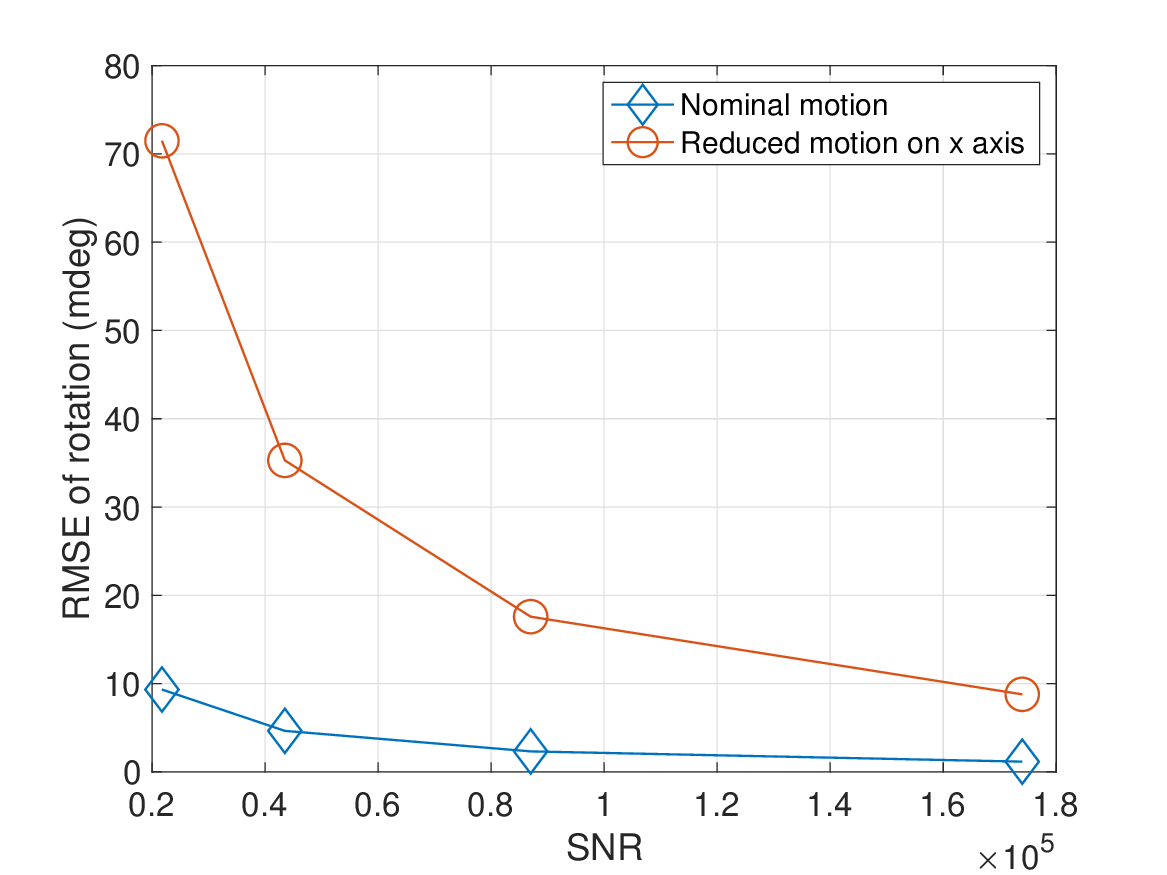}
\captionsetup{margin=0.5em}
\caption{
Rotation \acs{RMSE} with reduced velocity along $x$ axis, compared with \acs{RMSE} with nominal motion.
}
\label{fig:rmse_rotation_weak_axis}
\end{minipage}
\end{figure*}

In our simulations, we employ data from human motion (see Fig.~\ref{fig:gyro_motion}) to emulate the ideal measurements of gyroscope~$1$ sampled at $100\,\text{Hz}$. 
Then, we use the true rotational extrinsics $\mathbf{C}$ to generate the ideal data of gyroscope~$2$.
Next, we introduce intrinsics errors (in scale factors and skewness for some tests) and finally, we contaminate the gyroscopes' data [see (\ref{eq:omega_meas})] with additive noise $\mathbf n_i(k) \sim \mathcal N(\mathbf 0,\,\sigma^2\mathbf I)$,  
$\sigma = 0.1\, \si[per-mode=symbol]{\deg\per\second}$ and time-varying bias modelled as:
\begin{align}
    \mathbf b_i(k+1) = \mathbf b_i(k) + \boldsymbol\nu_i(k+1)
\end{align}
where $\boldsymbol\nu_i(k)\sim \mathcal N(\mathbf 0,\,\sigma_{\nu}^2 \mathbf I)$, $\sigma_{\nu} =57.3\,\si{\micro \deg\per\second}$.
%


%
%

Two key metrics include the \ac{RMSE} accuracy and \ac{SNR}:
%
\begin{align}
    \text{RMSE} \coloneqq \scalemath{0.9}{\biggl(\frac{1}{M}\sum_{i=1}^M e_i^2\biggr)^{\frac{1}{2}}} ~,~
    \text{SNR} = \scalemath{0.9}{\biggl(\frac{1}{\sigma^2}\sum_{k=1}^N \|{}^1\boldsymbol\omega_k\|^2 \biggr)^{\frac{1}{2}}} \label{eq:snr}
\end{align}
where $M$ is the number of Monte-Carlo runs.
The error for the $i$th Monte-Carlo run, $e_i$, is defined as follows, for orientation and intrinsics, respectively:
%
\begin{align}
      e_i &= \left\| \mathrm{Log}(\hat{\mathbf C} \mathbf C^\top) \right\|\\
      e_i 
      &= 
      \frac{1}{\sqrt{5}}
      \left\| \frac{1}{{}^1\hat s_x} 
      \left[\begin{array}{c}
      {}^1 \hat{\mathbf s} \\
      {}^2 \hat{\mathbf s}
      \end{array}
      \right]
      - 
      \frac{1}{{}^1 s_x}
      \left[\begin{array}{c}
      {}^1 {\mathbf s} \\
      {}^2 {\mathbf s}
      \end{array}
      \right]
      \right\|
\end{align}
where $\hat{\mathbf C}$ and $\mathbf C$ are the estimated and true rotations, and ${}^i\hat{\mathbf s} = [{}^i\hat s_x~{}^i\hat s_y~{}^i\hat s_z]^\top$ and
${}^i\mathbf s = [{}^i s_x~{}^i s_y~{}^i s_z]^\top$
are the estimated and true gyroscopes' scale factors for $i=1,2$.

%
Asymptotically, the \ac{RMSE} coincides with the square-root of error {\em covariance} for maximum-likelihood estimators \cite{kay1993fundamentals}.
Furthermore, a lower bound on the error covariance of the estimated orientation is derived in Appendix \ref{appendix:cov_snr}
\begin{align}
    \text{Tr}\bigl\{\mathbb E\{\mathbf P\} \bigr\} \geq 
 4.5 \bigl( \mathbb E\bigl\{ \text{SNR}^2 \} \bigr)^{-1}\,.
\end{align}
%
%
According to this expression, a lower bound on the \ac{RMSE} is expected to be approximately proportional to the inverse of \ac{SNR}, and hence to the
\begin{itemize}
    \item Inverse of the gyroscope velocity norm 
    \item Standard deviation of noise $\sigma$, and
    \item $1/\sqrt{N}$ assuming the data follow a stationary process.
\end{itemize}
%

\subsection{Performance with Known Skewness}

In order to validate the relationship between \ac{RMSE} and \ac{SNR}, we first consider the case without skewness error (i.e., the intrinsic matrices $\mathbf S_1$ and $\mathbf S_2$ are diagonal) and ``scale'' the \ac{SNR} of the true gyroscope data in Fig. \ref{fig:gyro_motion}, by factors of $0.5$, $1$, $2$, and $4$.
This is accomplished by appropriately scaling the velocities' norms, noise $\sigma$, or data sample size.
As evident from the results in
Fig.s~\ref{fig:rmse_rotation} and~\ref{fig:rmse_scale}, respectively, for the estimated rotation and scale factors, their \ac{RMSE}s are indeed proportional to the inverse of the \ac{SNR}.


Next, we turn our attention to the fact that imbalance in the rotational velocities along each of the three axes affects calibration accuracy.
%
To illustrate this, we purposely reduce the rates about the $x$ axis $10$ fold, while amplifying the ones around the $y$ and $z$ axes such that the \ac{SNR} remains the same.
As shown in Fig.~\ref{fig:rmse_rotation_weak_axis}, the rotational \ac{RMSE} increases significantly.
This is to be expected as the information matrix, and hence the covariance, depends on the individual axes' rotational velocities [see~\eqref{eq:Hk}].
For this reason, in practice, the measurements used for calibration should be selected so as to ensure sufficient motion along all three axes.

\subsection{Performance with Imperfect Skewness Calibration} \label{sec:exp_skewness}

Next, we investigate the impact of skewness errors.
%
In Appendix \ref{appendix:impact_of_skewness},
we derive the following closed-form expression for the rotation error due to skewness:
\begin{eqnarray}
    \tildebtheta \simeq \frac{1}{2} \left( \mathbf{w}_1 - \mathbf{C} \mathbf{w}_2 + \mathbf{n} \right) \label{rot_err_analytical}
\end{eqnarray}
where $\mathbf{n}$ is the LS error when solving for $\mathbf A$ from (\ref{eq:optimalA}), which goes to zero as the number of measurements increases, 
%
\begin{eqnarray}
\mathbf{w}_i \simeq \bigl[\, -\mathbf{\widetilde{S}}_i(2,3) ~~ \mathbf{\widetilde{S}}_i(1,3) ~ -\mathbf{\widetilde{S}}_i(1,2) \,\bigr]^\top  ,\, i=1, 2 \label{skewness_error}
\end{eqnarray}
and $\mathbf{\widetilde{S}}_i(1,2)$, $\mathbf{\widetilde{S}}_i(1,3)$, and $\mathbf{\widetilde{S}}_i(2,3)$ are the skewness errors in the corresponding intrinsics matrices.
According to~\eqref{rot_err_analytical}, the rotation error is asymptotically bounded by the skewness errors.
To assess this in simulation, 
the three upper off-diagonal elements in $\mathbf S_1$ and $\mathbf S_2$, are randomly drawn\footnote{The range of values used for skewness as well as for bias drift and additive noise are representative of MEMS gyroscopes' performance.} from $\mathcal{N}(0, \sigma_s^2)$ with $\sigma_s = 2.357\times 10^{-4}$.
As seen from the results depicted in Fig.s~\ref{fig:rmse_rotation} and~\ref{fig:rmse_scale}, there is a strong agreement between the estimated rotation error and that predicted from~\eqref{rot_err_analytical}.

Lastly, Fig.s \ref{fig:rmse_rotation} and \ref{fig:rmse_scale} confirm that the proposed {\em direct} LS algorithm achieves performance indistinguishable to that of a linearization-based iterative solver in Appendix \ref{appendix:iterative_solver}, without any iterations. 


\section{Experimental Results}

\begin{table}
\begin{center}
\caption{Orientation errors (\si{mdeg})}
\label{table_rot_error}
\begin{tabular}{ ccc }
\toprule 
$\theta_\text{x}$ & $\theta_\text{y}$ & $\theta_\text{z}$ \\
\midrule
23.52 & -72.82 &  82.39 \\
23.91 & -71.91  & 75.09 \\
26.48 & -72.02 &  77.28 \\
\bottomrule
\end{tabular}
\vspace{0.75em}
\caption{Scale factor errors (\%)}
\label{table_scale}
\begin{tabular}{ cccccc }
\toprule 
${}^1 s_x$ & ${}^1 s_y$ & ${}^1 s_z$ & ${}^2 s_x$ & ${}^2 s_y$ & ${}^2 s_z$ \\
\midrule
0.01 &  -0.04 & -0.01  & -0.01   &-0.08   &-0.05 \\
-0.01 &  -0.03 &  -0.03 &   0.01 &  -0.08    &  0.00 \\
-0.02  &  0.01  & -0.04 &   0.02  & -0.02  &  0.01 \\
\bottomrule
\end{tabular}
\end{center}
\end{table}

\subsection{Rigid Motion}

The proposed algorithm is employed for infield calibration of a device equipped with two IMUs.
In this process, we accumulate rotational velocity measurements till we reach an \ac{SNR} of $4\times 10^3$ per axis.
To evaluate performance, we use as \ac{GT} the parameters determined by a high-accuracy calibration station.
%
%
%
The errors in rotation and scale factors are shown in Tables ~\ref{table_rot_error} and~\ref{table_scale}, respectively.
In these tests, the global scale is resolved by employing~\eqref{opt_refine_scale} and considering as prior that the $x$ axes's scale factors remain close to $1$.
%
These results confirm the proposed infield algorithm's ability to achieve high calibration accuracy.



\subsection{Robustness to Device Flexibility}
 
Under fast motion, some devices may flex which violates the assumption of rigidity between the gyroscopes.
Using data from these time instants will negatively impact calibration performance.
In what follows, we present our approach for detecting and removing such measurements.

To better understand the effect of device flex we compute the residual, i.e., the difference in the two gyroscopes' measurements after accounting for intrinsics and extrinsics: 
\begin{align} \label{eq:residual}
    \mathbf r(k) 
    = 
    {}^1\breve{\boldsymbol\omega}_\text{m}(k)
    - \hat{\mathbf S}_1\hat{\mathbf C}\hat{\mathbf S}_2^{-1}
    {}^2\breve{\boldsymbol\omega}_\text{m}(k)\,.
\end{align}
As shown in Fig.~\ref{fig:residual_gyro}~(right), for rigid motions, the true rotational extrinsic is time-invariant, and the measurement residuals resemble Gaussian noise independent of the magnitude of rotation.
In contrast, when the device undergoes very fast motion and flexes, the true rotational extrinsic is time-variant, and the residuals become strongly correlated with the rotational velocity norm - see Fig.~\ref{fig:residual_gyro}~(left).




Hereafter, we present a simple but effective way to robustify the calibration algorithm 
against device flexibility.
First, we run the infield algorithm once to compute the residuals at each time step and their standard deviation $\sigma_r$.
Then, we remove gyroscope data whose residuals have magnitude greater than $3\sigma_r$, as well as measurements immediately before and after these time instants to account for hysteresis in the onset/settle of device flex.
%
%
Finally, we re-run the algorithm using the remaining gyroscope measurements.
As evident from the results shown in Table~\ref{table_remove_spike}, the rotation accuracy is greatly improved by following this approach.

\section{Conclusion}
\label{sec:conclusions}

In this paper, we address the problem of determining the {\em relative orientation} of two rigidly connected parts of a device that may have undergone plastic deformation. 
To do so, we considered data available from two 3-axial gyroscopes -- each attached on the corresponding part of interest -- and introduced a {\em direct} algorithm for determining the gyroscopes' extrinsics, as well as their scale factors, up to a global scale.
The main advantages of our algorithm is that it does not rely on specialized equipment or particular motions.
Instead, it can be employed infield without using any additional sensors and delivers high accuracy, comparable to that of a properly initialized iterative solver, in a single step.
As part of our main contributions, we analyzed the observability of the system comprising the two gyroscopes and showed that both extrinsics and scale factors (up to global scale) can be determined.
Furthermore, we assessed the impact of the gyroscopes' skewness error on the extrinsics accuracy.

%

Additionally, we validated the performance of the proposed algorithm both in simulation and real-world experiments, 
where established relations between the expected error and the SNR which is a function of motion profile, gyro noise, and the size of the data used,
and defined metrics for determining when there is sufficient information for accurate estimation and when calibration should be avoided due to device flexibility. Lastly, we introduced a simple but effective method for determining device flex and removing measurements that negatively affects estimation accuracy.


\begin{figure}[t]
\centering
\includegraphics[width=1\columnwidth]{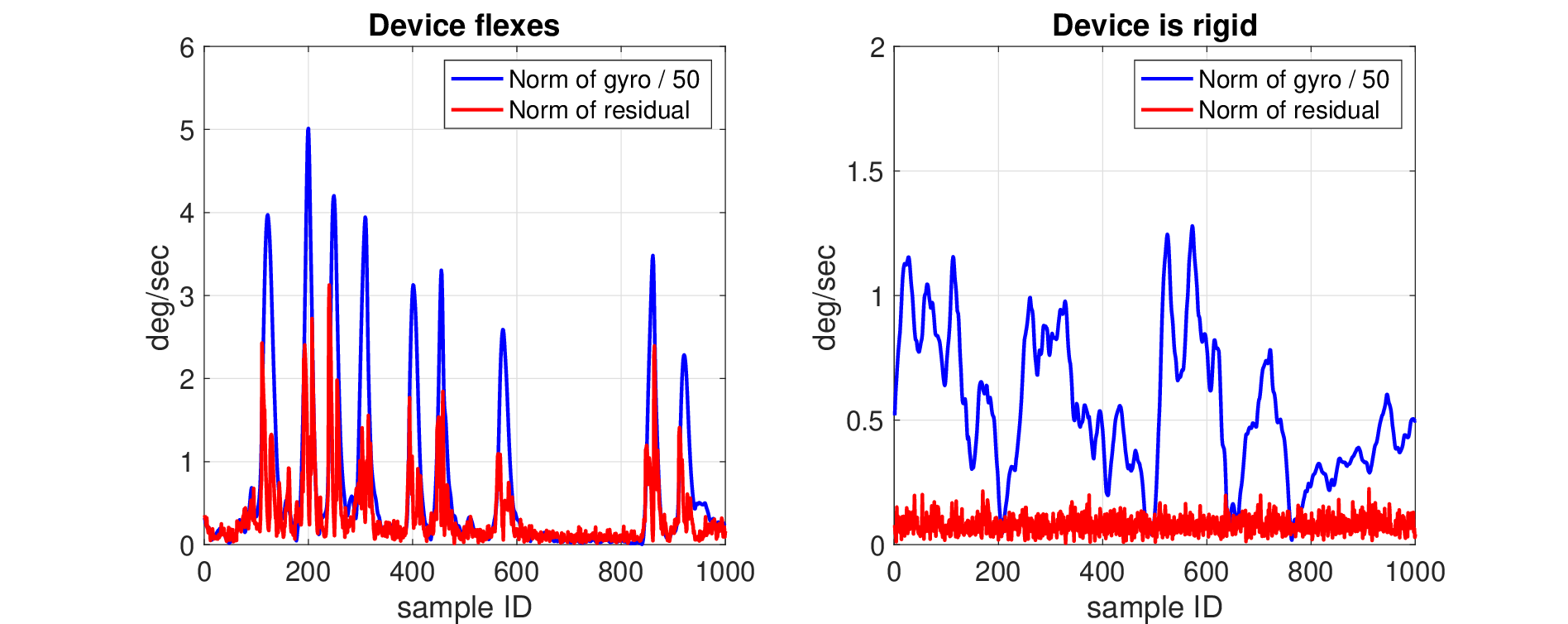}
\caption{Norm of residual vs. scaled norm of gyro data.}
\label{fig:residual_gyro}
\end{figure}

\begin{table}[th]
\begin{center}
\caption{Orientation errors \si{mdeg} before and after removing measurements via the proposed \textit{mitigation} scheme.}
\label{table_remove_spike}
\begin{tabular}{ cccc }
\toprule 
& $\theta_\text{x}$ & $\theta_\text{y}$ & $\theta_\text{z}$  \\
\midrule
Before &  19.24 & 985.92 & 401.91 \\
After  &  -28.07 & -14.05 & -3.09 \\
\bottomrule
\end{tabular}
\end{center}
\end{table}

\appendices

\section{Proof of Lemma \ref{lem1}}  \label{lem1_proof}
Note that the column vectors of $\bbOmega_2$ in (\ref{eq:Omega_i}) are \textit{re-centered} measurements ${}^{2}\breve{\bbomega}_m(k)$ as in (\ref{eq:omega_shift}).
To recover $\mathbf{A}$ from (\ref{eq:optimalA}), we need $\si{rank}(\bbOmega_2) = 3$.
Suppose only three measurements ${}^{2}\bbomega_m(k)$ spanning $3$ d.o.f., are available.
Also, let $\mathbf P_N$ be an $N\times N$ matrix whose elements are given by 
\begin{eqnarray}
\mathbf P_N(i,j) = 
    \begin{cases}
    (N-1)/N, &i=j \\
    -1/N, &\text{otherwise}\,.
    \end{cases}
\end{eqnarray}
Then $\bbOmega_2$ in (\ref{eq:Omega_i}) for $N=3$ can be written as
\begin{eqnarray} \label{eq:CfromA}
\bbOmega_2 &=
\left[\,
{}^{2}{\bbomega}_m(1) \;\;
{}^{2}{\bbomega}_m(2) \;\;
{}^{2}{\bbomega}_m(3)\,\right] \mathbf P_3\,.
\end{eqnarray}
Since the every row or column of $\mathbf P_3$ sum up to $0$, $\si{rank}(\mathbf P) = 2$. 
This implies $\si{rank}(\bbOmega_2) \le 2$.
For this reason, a $4$-th measurement ${}^{2}\bbomega_m(4)$, different from the previous linearly independent ones, is necessary. 

Since ${}^{2}\bbomega_m(k)$, $k = 1,2,3$, span $3$ d.o.f.,
\begin{eqnarray}
    {}^{2}\bbomega_m(4) = \alpha_1 {}^{2}\bbomega_m(1) + \alpha_2 {}^{2}\bbomega_m(2) + \alpha_3 {}^{2}\bbomega_m(3)
\end{eqnarray}
for some $\alpha_1$, $\alpha_2$, and $\alpha_3$. 
Then $\bbOmega_2$ in (\ref{eq:Omega_i}) for $N=4$ can be written as
\begin{align}
    \bbOmega_2 &= \left[\,
{}^{2}{\bbomega}_m(1) \;\;
{}^{2}{\bbomega}_m(2) \;\;
{}^{2}{\bbomega}_m(3) \;\;
{}^{2}{\bbomega}_m(4)\,\right] \mathbf P_4
    \\
    &=
\frac{1}{4} \left[\,
{}^{2}{\bbomega}_m(1) \;\;
{}^{2}{\bbomega}_m(2) \;\;
{}^{2}{\bbomega}_m(3)\,\right] 
\left[\,\mathbf D\;\;\mathbf d\,\right] \label{eq:omega_D_d}
\end{align}
where $\mathbf d = \left[\,-1+3\alpha_1\;\;-1+3\alpha_2\;\;-1+3\alpha_3\,\right]^\top$ and
\begin{align}
    \mathbf D = \left[\begin{array}{ccc}
    3-\alpha_1 & -1-\alpha_1 & -1-\alpha_1 \\
    -1-\alpha_2 & 3-\alpha_2 & -1-\alpha_2 \\
    -1-\alpha_3 & -1-\alpha_3 & 3-\alpha_3
    \end{array}\right]\,.
\end{align}
The \textit{re-centered} measurements are linearly dependent since
\begin{eqnarray}
    {}^{2}\breve\bbomega_m(1) + {}^{2}\breve\bbomega_m(2) + {}^{2}\breve\bbomega_m(3) + {}^{2}\breve\bbomega_m(4) = \mathbf 0_{3\times 1}\,.
\end{eqnarray}
Thus, to have $\si{rank}(\bbOmega_2) = 3$, its first three columns ${}^{2}\breve\bbomega_m(k)$, $k=1,2,3$, need to span $3$ d.o.f., 
which requires $\mathbf D$ in (\ref{eq:omega_D_d}) to be invertible. Since
\begin{eqnarray}
    \det(\mathbf D) = 16 \cdot (1 - \alpha_1 -\alpha_2 - \alpha_3)
\end{eqnarray}
$\alpha_1 +\alpha_2 + \alpha_3 \neq 1$ is required for $\mathbf D$ to be invertible.


\section{Lower Bound on Error Covariance} \label{appendix:cov_snr}
Consider the simplified case of zero bias and perfect intrinsics where only rotation is estimated. The residual for the $k$th pair of measurements is:
\begin{align}
	\mathbf r(k) 
	&= 
	{}^1{\boldsymbol\omega}_\text{m}(k) 
	- \hat{\mathbf C}\, {}^2{\boldsymbol\omega}_\text{m}(k) \label{residual}
 \\
	&\simeq
 {}^1 \boldsymbol\omega(k) + \mathbf n_1(k) \nonumber
        - (\mathbf I + \skewm{\tilde{\boldsymbol\theta}} ) \mathbf C 
	\bigl( {}^2 \boldsymbol\omega(k) + \mathbf n_2(k) \bigr) \nonumber \\
	&\simeq
	\skewm{ {}^1 \boldsymbol\omega(k) }
	\tilde{\boldsymbol\theta} 
	+  \mathbf n_{1}(k) + \mathbf C \,\mathbf n_{2}(k) \nonumber 
\end{align}
where $\hat{\mathbf C} \simeq (\mathbf I + \skewm{\tilde{\boldsymbol\theta}} ) \mathbf C$ with $\tilde{\boldsymbol\theta}$ being the rotation error. 
%
Furthermore, let $\text{Cov}\{\mathbf n_{12}\} = \sigma^2 \mathbf I$. Then, the information matrix from the $k$th measurement is:
\begin{align}
    \hspace{-1mm}\mathbf H_k
    &= \frac{1}{\sigma^2} \skewm{ {\boldsymbol\omega}(k) }^\top
    \skewm{ {\boldsymbol\omega}(k) } 
    \nonumber
    \\
    &=
    \frac{1}{\sigma^2}
    \left[\hspace{-1mm} 
    \begin{array}{ccc}
       \omega_{k,2}^2 + \omega_{k,3}^2  & -\omega_{k,1}\omega_{k,2} & -\omega_{k,1}\omega_{k,3} \\
       -\omega_{k,1}\omega_{k,2}  & \omega_{k,1}^2 + \omega_{k,3}^2 & -\omega_{k,2}\omega_{k,3} \\
        -\omega_{k,1}\omega_{k,3} & -\omega_{k,2}\omega_{k,3} & \omega_{k,1}^2 + \omega_{k,2}^2
    \end{array}
    \hspace{-1mm}\right] \label{eq:Hk}
\end{align}
where we have dropped the superscript $1$ for clarity,
and $\omega_{k,i}$ is the $i$th element of $\bbomega(k)$.

Let $\mathbf P$ and $\mathbf H$ be the $3\times 3$ error covariance and information matrix, respectively. We have
\begin{align}
    \mathbf P = \mathbf H^{-1} = \biggl(\sum_{k=1}^N \mathbf H_k \biggr)^{-1}. \label{eq:covP}
\end{align}
Furthermore, suppose the rotation velocities are zero mean and uncorrelated across axes. Then, from~\eqref{eq:covP} and~\eqref{eq:Hk}, using Jensen's inequality \cite{boyd2004convex}, and assuming that $\bbomega(k)$'s are zero-mean and independent among the $3$ axes, we have
\begin{align}
    \mathbb E\{\mathbf P\}
    &\succeq
    \biggl(\sum_{k=1}^N \mathbb E\{\mathbf H_k\} \biggr)^{-1} \nonumber
    =
    \sigma^2\diag\{E_{23}^{-1}, E_{13}^{-1}, E_{12}^{-1}\} \nonumber
\end{align}
where $\mathbf A \succeq \mathbf B$ means $\mathbf A - \mathbf B$ is positive semi-definite, and
$E_{ij} := \sum_{k=1}^N \mathbb E\{ \omega_{k,i}^2 + {\omega_{k,j}^2} \}$.
Therefore:
%
\begin{align}
    \text{Tr}\bigl\{\mathbb E\{\mathbf P\} \bigr\}
    &\ge \sigma^2 (E_{23}^{-1} + E_{13}^{-1} + E_{12}^{-1})
    \\
    &\ge 9\sigma^2 (E_{23} + E_{13} + E_{12})^{-1} \label{eq:lower_bound_2}
    \\
    &= 4.5 \bigl( \mathbb E\bigl\{ \text{SNR}^2 \} \bigr)^{-1} \label{eq:lower_bound_P_final}
\end{align}
where (\ref{eq:lower_bound_P_final}) uses (\ref{eq:snr}), and (\ref{eq:lower_bound_2}) uses Jensen's inequality and that $f(x) = 1/x$ is a convex function.

\section{Iterative Solver} \label{appendix:iterative_solver}
The intrinsic error is modeled as
\begin{align} 
	\mathbf S_i &= \hat{\mathbf S}_i + \tilde{\mathbf S}_i 
	= \hat{\mathbf S}_i ( \mathbf I + \tilde{\mathbf S}_i')
\end{align}
in which $\hat{\mathbf S}_i$ is the current estimate of $\mathbf S_i$, $\tilde{\mathbf S}_i$ is the estimation error,
and $\tilde{\mathbf S}_i' = \hat{\mathbf S}_i^{-1} \tilde{\mathbf S}_i$ is the percentage formulation of the estimation error estimated by the iterative solver.

Multiplying both sides of (\ref{eq:omega_meas}) with $\hat{\mathbf S}_i^{-1}$, we have
\begin{align}
\hspace{-2mm}
	{}^i{\boldsymbol\omega}_\text{m}'(k) 
	&\coloneqq 
	\hat{\mathbf S}_i^{-1}  {}^i \boldsymbol\omega_\text{m}(k) 
	= 
	(\mathbf I + \tilde{\mathbf S}_i') {}^i \boldsymbol\omega(k) + \mathbf b_i' + \mathbf n_i'(k)
\end{align}
where $\mathbf b_i' \coloneqq \hat{\mathbf S}_i^{-1} \mathbf b_i$ and $\mathbf n_i'(k) \coloneqq \hat{\mathbf S}_i^{-1} \mathbf n_i(k)$. Furthermore, the true rotation is approximated as $\mathbf C \simeq (\mathbf I - \skewm{\tilde{\boldsymbol\theta}}) \hat{\mathbf C}$.
Then the residual can be defined as
\begin{align}
	\mathbf r(k) 
	&= 
	{}^1{\boldsymbol\omega}_\text{m}'(k) 
	- \hat{\mathbf C}\, {}^2{\boldsymbol\omega}_\text{m}'(k) \label{residual}
 \\
	&\simeq
	(\mathbf I + \tilde{\mathbf S}_1') {}^1 \boldsymbol\omega(k) + \mathbf b_1' + \mathbf n_1'(k) \nonumber
        \\
	&- (\mathbf I + \skewm{\tilde{\boldsymbol\theta}} ) \mathbf C 
	\bigl( (\mathbf I + \tilde{\mathbf S}_2') {}^2 \boldsymbol\omega(k) + \mathbf b_2' + \mathbf n_2'(k) \bigr) \nonumber \\
	&=
	\skewm{ {}^1 \boldsymbol\omega(k) }
	\tilde{\boldsymbol\theta} 
	+ \boldsymbol\Gamma\, {}^1 \boldsymbol\omega(k)
	+ \mathbf b_{12} +  \mathbf n_{12}(k) \nonumber 
\end{align}
where 
$\mathbf b_{12} \coloneqq \mathbf b_1' - (\mathbf I + \skewm{\tilde{\boldsymbol\theta}})  \hat{\mathbf C} \,\mathbf b_2'$,
$\mathbf n_{12}(k) \coloneqq \mathbf n_1'(k) -  \hat{\mathbf C}\, \mathbf n_2'(k)$, 
and
\begin{align}
	\boldsymbol\Gamma 
	&\coloneqq 
	\tilde{\mathbf S}_1' -  {\mathbf C}\, \tilde{\mathbf S}_2'\, {\mathbf C}^\top  \label{Gamma}
 \end{align}
which is symmetric since $\tilde{\mathbf S}_1'$ and $\tilde{\mathbf S}_2'$ are diagonal. Let
\begin{align}
    \boldsymbol\gamma 
	&\coloneqq 
	\left[\begin{array}{cccccc}
		\gamma_{11} & \gamma_{12} & \gamma_{13} & \gamma_{22} & \gamma_{23} & \gamma_{33}
	\end{array}\right]^\top 
\end{align}
with $\gamma_{ij}$ being the $(i,j)$th element of $\boldsymbol\Gamma$,
We have
\begin{align}
    \mathbf r(k) 
    &=
	\skewm{ {}^1 \boldsymbol\omega(k)  }
	\tilde{\boldsymbol\theta} 
	+ \boldsymbol\Omega(k)\boldsymbol\gamma + \mathbf b_{12} + \mathbf n_{12}(k) \nonumber \\
	&=
	\mathbf J(k) \mathbf x + \mathbf b_{12} + \mathbf n_{12}(k)  \label{residual_linear}
\end{align}
where
\begin{align}
    {\boldsymbol\Omega}(k)
	&=
	\left[\begin{array}{cccccc}
	\omega_{1,k} & \omega_{2,k} & \omega_{3,k} & 0 & 0 & 0 \\
	0 & \omega_{1,k} & 0 &  \omega_{2,k} & \omega_{3,k} & 0 \\
	0 & 0 & \omega_{1,k} & 0 & \omega_{2,k} & \omega_{3,k}
	\end{array}\right] 
	\\
	\mathbf J(k) 
	&= 
	\left[\begin{array}{cc}
	 \skewm{ {}^1 \boldsymbol\omega(k)  } & \boldsymbol\Omega(k)
	\end{array}\right] \label{eq:Jk}
 \\
 \mathbf x &= \left[\begin{array}{cc}
	\tilde{\boldsymbol\theta}^\top & \boldsymbol\gamma^\top
	\end{array}\right]^\top
\end{align}
in which $\omega_{1,k}$, $\omega_{2,k}$, and $\omega_{3,k}$ are the $3$ elements of ${}^1\boldsymbol\omega(k)$.
Furthermore, since $\hat{\mathbf S}_i \simeq \mathbf I$ and $\hat{\mathbf C}$ is an orthonormal matrix, we have
$\text{Cov}\{\mathbf n_{12}\} \simeq \sigma^2 \mathbf I$ with $\sigma^2 = \sigma_1^2 + \sigma_2^2$.

Using (\ref{residual_linear}), a least-squares problem can be formulated with objective function given by
\begin{align}
	f(\mathbf x,\,\mathbf b_{12}) = \frac{1}{2 \sigma^2}\sum_{k=1}^N \| \mathbf r(k) - \mathbf J(k)\mathbf x - \mathbf b_{12} \|^2\,. \label{obj_fun}
\end{align}
To find the optimal solution, first solve for $\mathbf b_{12}$ as
\begin{align}
	\frac{\partial f}{\partial \mathbf b_{12}} = \mathbf 0 
	\Rightarrow
	\mathbf b_{12}^*(\mathbf x)
	= \frac{1}{N} \sum_{k=1}^N \bigl(\mathbf r(k) - \mathbf J(k) \mathbf x\bigr)
	= \bar{\mathbf r} - \bar{\mathbf J}\mathbf x\,. \nonumber
\end{align}
Substituting $\mathbf b_{12} = \mathbf b_{12}^*(\mathbf x)$ into (\ref{obj_fun}), 
we obtain an objective function about $\mathbf x$ as
\begin{align}
	f'(\mathbf x) = \frac{1}{2\sigma^2}\sum_{k=1}^N \| \breve{\mathbf r}(k) - \breve{\mathbf J}(k)\mathbf x \|^2 \label{eq:obj_iter_cost_2}
\end{align}
where
\begin{align}
	\breve{\mathbf r}(k) 
	&=
	{\mathbf r}(k) - \bar{\mathbf r} 
	=
	{}^1\breve{\boldsymbol\omega}_\text{m}'(k) 
	- \hat{\mathbf C}\, {}^2\breve{\boldsymbol\omega}_\text{m}'(k) \label{equ:residual}
	\\
	\breve{\mathbf J}(k)
        &=
	{\mathbf J}(k) - \bar{\mathbf J} 
	= 
	\left[\begin{array}{cc}
	 \skewm{ {}^1 \breve{\boldsymbol\omega}(k) } & \breve{\boldsymbol\Omega}(k)
	\end{array}\right] \label{Jk_iterative}
\end{align}
in which (\ref{equ:residual}) uses (\ref{residual}), 
\begin{align}
	\breve{\boldsymbol\Omega}(k)
	&=
	\left[\begin{array}{cccccc}
	\breve\omega_{1,k} & \breve\omega_{2,k} & \breve\omega_{3,k} & 0 & 0 & 0 \\
	0 & \breve\omega_{1,k} & 0 &  \breve\omega_{2,k} & \breve\omega_{3,k} & 0 \\
	0 & 0 & \breve\omega_{1,k} & 0 & \breve\omega_{2,k} & \breve\omega_{3,k}
	\end{array}\right] 	\nonumber
\end{align}
where $\breve\omega_{1,k}$, $\breve\omega_{2,k}$, and $\breve\omega_{3,k}$ are the $3$ elements of ${}^1\breve{\boldsymbol\omega}(k)$.
Note that 
the Jacobian $\breve{\mathbf J}(k)$ in (\ref{Jk_iterative}) can be computed using the measurements.

Then, the iterative solver using, e.g., Gauss-Newton method, solves a least squares problem for $\mathbf x$ in each iteration with objective function given by (\ref{eq:obj_iter_cost_2}). 
After that, the diagonal elements of $\tilde{\mathbf S}_1'$ and $\tilde{\mathbf S}_2'$, ${}^1\tilde{\mathbf s}'$ and ${}^2\tilde{\mathbf s}'$, can be recovered from (\ref{Gamma}) for $\mathbf C = \hat{\mathbf C}$. 
Since $\boldsymbol\Gamma$ is symmetric, (\ref{Gamma}) provides $6$ linear equations about $\tilde{\mathbf s}_1'$ and $\tilde{\mathbf s}_2'$
\begin{align}
    \left[\begin{array}{c}
    \gamma_{11} \\ \gamma_{12} \\ \gamma_{13} \\ \gamma_{22} \\ \gamma_{23} \\ \gamma_{33}
    \end{array}\right]
    =
    \underbrace{\left[\begin{array}{cccc}
    1 & 0 & 0 & -\mathbf c_1 \odot \mathbf c_1 \\
    0 & 0 & 0 & -\mathbf c_1 \odot \mathbf c_2 \\
    0 & 0 & 0 & -\mathbf c_1 \odot \mathbf c_3 \\
    0 & 1 & 0 & -\mathbf c_2 \odot \mathbf c_2 \\
    0 & 0 & 0 & -\mathbf c_2 \odot \mathbf c_3 \\
    0 & 0 & 1 & -\mathbf c_3 \odot \mathbf c_3
    \end{array}\right]}_{\boldsymbol\Phi}
    \left[\begin{array}{c}
    {}^1\tilde{s}_x \\ {}^1\tilde{s}_y \\ {}^1\tilde{s}_z \\ {}^2\tilde{s}_x \\ {}^2\tilde{s}_y \\ {}^2\tilde{s}_z
    \end{array}\right]
\end{align}
where $\mathbf c_i$ is the $i$th row of $\mathbf C$, and $\odot$ means element-wise multiplication. 
Since $\mathbf C\in \mathrm{SO(3)}$, the sum of the $3$ elements of $\mathbf c_i \odot \mathbf c_j$ is $1$ for $i=j$, and $0$ otherwise.
As a result, the column vectors of $\boldsymbol\Phi$ are linearly dependent, and 
at most $5$ d.o.f. of the scale factors can be recovered, which agrees with Theorem \ref{thm:observability}. 
Thus, we fix the $1$st element of ${}^1\tilde{\mathbf s}'$ as zero and solve for the remaining $5$ scale factor errors. 
For singularity cases in Lemma \ref{lem2} and \ref{lem3}, more elements in ${}^1\tilde{\mathbf s}'$ and ${}^2\tilde{\mathbf s}'$ need to be fixed as zero.
Typically the algorithm converges in fewer than $5$ iterations.

\section{Impact of Skewness Error} \label{appendix:impact_of_skewness}

To assess the impact of intrinsic errors on the accuracy of rotation estimation, we rewrite \eqref{eq:defineA} as
$\mathbf{C} = \mathbf{S}_1^{-1} \mathbf{A}   \mathbf{S}_2$,
%
which holds regardless of the structure of $\mathbf{S}_1$ and $\mathbf{S}_2$.
Similarly, the estimated value of $\mathbf C$ is given by
\begin{eqnarray} \label{eq:C_hat}
\mathbf{\hat{C}} &= \mathbf{\hat{S}}_1^{-1} \mathbf{\hat{A}}   \mathbf{\hat{S}}_2\,.
\end{eqnarray}

Define the orientation error from the following relations
\begin{eqnarray} \label{eq:small_angle_approx_C_repeat}
\rotmh \simeq \bigl( \eye3 + \skewm{ \tildebtheta}\bigr) \rotm 
~\Leftrightarrow~
%
\rotm \simeq \bigl( \eye3 - \skewm{ \tildebtheta }\bigr) \rotmh
\end{eqnarray}
from which we have
\begin{eqnarray} \label{eq:skew_symmetric_error}
\skewm{ \tildebtheta} \simeq \frac{1}{2} \bigl ( \rotmh \cdot \rotm^\top  - \rotm  \cdot \rotmh^\top \bigr)\,.
\end{eqnarray}
Next, for a general variable $\mathbf X$, define the relation between the estimate, $\mathbf{\hat{X}}$, the error, $\mathbf{\widetilde{X}}$, and the true value of $\mathbf{X}$
\begin{eqnarray}
\mathbf{\hat{X}}   &:= \mathbf{X}  - \mathbf{\widetilde{X}}\,. \label{eq:S1}
\end{eqnarray}
%
From \eqref{eq:S1} for $\mathbf X = \mathbf S_1$, we have:
\begin{eqnarray} \label{eq:invS1}
 \mathbf{\hat{S}}_1 ^{-1} = \bigl( \mathbf{S}_1 -  \mathbf{\widetilde{S}}_1  \bigr)^{-1}
 \simeq \mathbf{S}_1^{-1} + \mathbf{S}_1^{-1}  \mathbf{\widetilde{S}}_1  \mathbf{S}_1^{-1}\,.
\end{eqnarray}

Substituting \eqref{eq:invS1} and \eqref{eq:S1} for $\mathbf X = \mathbf A$ and $\mathbf S_2$ in \eqref{eq:C_hat} and employing the small error approximation yields:
%
\begin{align} 
\mathbf{\hat{C}} &\simeq
\bigl( \mathbf{S}_1^{-1} + \mathbf{S}_1^{-1}  \mathbf{\widetilde{S}}_1  \mathbf{S}_1^{-1} \bigr)
\bigl(  \mathbf{A}  - \mathbf{\widetilde{A}} \bigr)
\bigl(  \mathbf{S}_2 -  \mathbf{\widetilde{S}}_2 \bigr) \nonumber \\ 
&\simeq
\mathbf{S}_1^{-1} \mathbf{A}   \mathbf{S}_2  
-  \mathbf{S}_1^{-1} \mathbf{A}    \mathbf{\widetilde{S}}_2 
- \mathbf{S}_1^{-1}  \mathbf{\widetilde{A}}   \mathbf{S}_2   
+ \mathbf{S}_1^{-1}  \mathbf{\widetilde{S}}_1 \mathbf{S}_1^{-1} \mathbf{A}   \mathbf{S}_2 \nonumber  \\
&=
\mathbf{C}  - 
\mathbf{C} \mathbf{S}_2 ^{-1}  \mathbf{\widetilde{S}}_2 
-  \mathbf{S}_1^{-1} \mathbf{\widetilde{A}}   \mathbf{S}_2  
+ \mathbf{S}_1^{-1} \mathbf{\widetilde{S}}_1 \mathbf{C} \nonumber
\end{align}
and thus
\begin{align} \label{eq:Cerror}
\hspace{-2mm}
\rotmh \cdot \rotm^\top  
&\simeq
\eye{3} +
 \mathbf{S}_1^{-1} \mathbf{\widetilde{S}}_1  
 - \mathbf{C} \mathbf{S}_2 ^{-1}  \mathbf{\widetilde{S}}_2  \mathbf{C}^\top
-  \mathbf{S}_1^{-1} \mathbf{\widetilde{A}}   \mathbf{S}_2    \mathbf{C}^\top
\end{align}
Substituting (\ref{eq:Cerror}) in~\eqref{eq:skew_symmetric_error} results in:
\begin{align} 
\skewm{ \tildebtheta} &\simeq 
%
%
%
%
\frac{1}{2} \left( 
\mathbf{W}_1 - \mathbf{C}  \mathbf{W}_2   \mathbf{C}^\top + \mathbf{N} 
\right) \label{eq:skew_symmetric_error_2}
\end{align}
where we have defined
\begin{eqnarray} \label{eq:skew_matrices}
\mathbf{W}_i &:=& \mathbf{S}_i^{-1} \mathbf{\widetilde{S}}_i - \bigl(  \mathbf{S}_i^{-1} \mathbf{\widetilde{S}}_i  \bigr)^\top  ~~~~~,\, i=1,2 \\
\mathbf{N} &:=&   \bigl( \mathbf{S}_1^{-1} \mathbf{\widetilde{A}}   \mathbf{S}_2    \mathbf{C}^\top \bigr)^\top  - \mathbf{S}_1^{-1} \mathbf{\widetilde{A}}   \mathbf{S}_2    \mathbf{C}^\top\,.
\label{eq:skew_noise}
\end{eqnarray}

Note that $\mathbf{W}_1$, $\mathbf{W}_2$ (and hence $\mathbf{C}  \mathbf{W}_2   \mathbf{C}^\top$), as well as $\mathbf{N}$ are all skew symmetric and so is the left hand-side of~\eqref{eq:skew_symmetric_error_2}.
With $\mathbf{w}_1$, $\mathbf{w}_2$, and $\mathbf{n}$ being the vectors generating $\mathbf{W}_1$, $\mathbf{W}_2$, and $\mathbf N$, respectively, and using 
$\mathbf C \skewm{\mathbf w_2} \mathbf C^\top = \skewm{ \mathbf C \mathbf w_2}$, we obtain (\ref{rot_err_analytical}).
Lastly, since $\mathbf{S}_i$ is close to identity matrix, and consider $\mathbf{\widetilde{S}}_i$ as upper triangular matrix with off-diagonal elements corresponding to skewness errors, we have (\ref{skewness_error}).
%
%

{
\vspace{0.05cm}
\bibliographystyle{packages/IEEEtran}
\bibliography{libraries/extra,libraries/related,libraries/rpng}
}

\end{document}